\documentclass[runningheads]{llncs}
\usepackage[T1]{fontenc}
\usepackage{graphicx}
\usepackage{float}
\usepackage{amsmath, amssymb}
\usepackage{bm}
\usepackage{enumitem}
\usepackage{url}

\begin{document}
\title{Post-Hoc Understanding of Metaphor Processing in Decoder-Only Language Models via Conditional Scale Entropy\thanks{Research reported in this publication was supported by the National Institute of General Medical Sciences of the National Institutes of Health under Award Number P20GM104420.}}
\titlerunning{Post-Hoc Understanding of Metaphor Processing via CSE}
%
\author{
Lawhori Chakrabarti\inst{1} \and
Jennifer Johnson-Leung\inst{2} \and
Bert Baumgaertner\inst{3} \and
Aleksandar Vakanski \inst{4} \and
Min Xian \inst{5} \and
Boyu Zhang \inst{1}
}
\authorrunning{L. Chakrabarti et al.}
%
\institute{Department of Computer Science, University of Idaho, Moscow ID 83844, USA \and
Department of Mathematics and Statistical Science, University of Idaho, Moscow ID 83844, USA \and
Department of Politics and Philosophy, University of Idaho, Moscow ID 83844, USA \and
Department of Nuclear Engineering and Industrial Management, University of Idaho, Idaho Falls ID 83402, USA \and
Department of Computer Science, University of Idaho, Idaho Falls ID 83402, USA
}
\maketitle              
\begin{abstract}
Metaphor requires a language model to resolve a token whose
contextual meaning diverges from its basic literal sense. Understanding
how transformer models organize this reinterpretation across depth
remains an open problem in mechanistic interpretability. We introduce
conditional scale entropy (CSE), a wavelet-derived measure of how
broadly transformer computation engages across frequency scales at
each layer position. Two theorems establish that CSE is invariant to
update magnitude, isolating the structural pattern of updates from
their intensity. Using CSE, we find that metaphorical tokens produce
significantly higher spectral breadth than literal tokens at
contiguous layer positions on every decoder-only architecture tested,
from 124M to 20B parameters (GPT-2 family, LLaMA-2 7B, GPT-oss
20B). The effect survives cluster-based permutation correction,
recurs in the early-to-mid relative depth range across models, and
converges with an independent analysis of 200 naturalistic VUA
pairs. Specificity controls further show that the effect is not
explained by semantic complexity or by matched propositional content.
These results identify multi-scale coordination as a consistent
signature of metaphorical language processing in the decoder-only
architectures examined, and establish CSE as a principled tool for
characterizing cross-depth structure in transformers.

\end{abstract}

\section{Introduction}
\label{sec:introduction}

Transformer language models process metaphorical language with
notable proficiency~\cite{wachowiak2023does}, yet the internal
mechanisms underlying this ability remain poorly understood. When a
model encounters ``The lawyer is a shark,'' the target token must be
reinterpreted from its literal referent to a figurative predicate.
We observe the input and the output; the intermediate computation
across the model's layers constitutes a black box. This paper asks a
specific question: does metaphorical processing organize the model's
internal computation differently from literal processing, and if so,
how can this difference be measured?

The residual trajectory~\cite{elhage2021mathematical} of a target token is a discrete sequence of hidden states indexed by layer depth. We construct a continuous interpolant from this sequence and apply the Continuous Wavelet Transform~\cite{mallat2009wavelet} (CWT) to obtain a scalogram that resolves energy at each layer position across frequency scales. From this scalogram we derive the conditional scale entropy (CSE), a quantity that measures how broadly computation engages across scales at each depth position. Two theorems establish that CSE is invariant to update magnitude and responds only to the structural pattern of updates across layers, separating how computation is organized from how intense it is.

Experiments across five architectures (124M to 20B parameters) and
two stimulus sets (25 controlled minimal pairs, 200 naturalistic VUA
pairs~\cite{steen2010method}) show that metaphorical tokens produce
significantly higher CSE at contiguous layer positions on every
model tested, surviving cluster-based permutation
correction~\cite{maris2007nonparametric}. This invariance is
important empirically: among the spectral quantities examined, CSE
is the only one that remains stable across architectures and
stimulus regimes. Specificity controls confirm that the elevation is
absent for literal sentences differing in semantic or syntactic
complexity, and a matched-meaning triples experiment further shows
that the effect tracks figurative mode of expression rather than
propositional content alone. The framework is mathematically
general; metaphor serves as the proving ground.

This paper makes three contributions: a wavelet-based
framework for analysing residual-stream trajectories across
transformer depth with the conditional scale entropy (CSE) as a
layer-resolved measure of spectral breadth; a theoretical
characterisation showing that CSE is invariant to update
magnitude and ordered by majorisation of the normalised scale
distribution; and empirical validation across five decoder-only
architectures, naturalistic VUA data, and targeted specificity
controls, where CSE proves the most stable spectral signature
of metaphorical processing.

Proofs of the theoretical results, the full controlled stimulus
list, and additional supporting analyses are deferred to an
extended version of this paper.

\section{Related Work}

\paragraph{Metaphor identification in NLP.}
Computational metaphor research has predominantly addressed
identification and interpretation rather than the internal
processing dynamics of neural language models. Early systems framed
the task as contextual sequence labelling
\cite{gao2018neural} or grounded it in selectional-preference
violation \cite{mao2019end}. Pretrained transformers
substantially advanced the state of the art (for instance, MelBERT
\cite{choi2021melbert} pairs contextualised late interaction with
the Metaphor Identification Procedure
\cite{steen2010method}), and more recent work asks whether
generative LLMs can recover conceptual metaphor mappings under
prompting~\cite{wachowiak2023does}. From cognitive linguistics,
the Graded Salience Hypothesis \cite{giora1997graded} and the
Career of Metaphor theory \cite{bowdle2005career} predict that
novel metaphors demand structural alignment between source and
target domains, whereas conventional metaphors are resolved by
categorisation. These predictions motivate a computational measure
of processing structure that operates inside the model rather
than at the input--output level.

\paragraph{Transformer representations across depth.}
A complementary line of work asks how linguistic structure is
organised within transformer models: BERT recapitulates the
classical NLP pipeline across depth~\cite{tenney2019bert};
contextualised representations grow increasingly context-specific
with depth~\cite{ethayarajh2019contextual}; metaphorical knowledge
is recoverable from pretrained representations with the strongest
probing signal at middle layers~\cite{aghazadeh2022metaphors};
and probe accuracy alone is insufficient for representational
claims~\cite{voita2020information}. Mechanistic accounts shift
the question from what a layer encodes to how it
transforms the residual stream
\cite{elhage2021mathematical,geva2022ffn}, and Geshkovski et
al.~\cite{geshkovski2025mathematical} treat depth as continuous
time, motivating the analysis of the hidden-state sequence as a
discrete signal indexed by depth.

\paragraph{Spectral analysis of transformers.}
Frequency-domain methods have illuminated deep-network behaviour:
spectral bias in training~\cite{rahaman2019spectral},
information-bottleneck dynamics across layers~\cite{shwartzziv2017opening},
wavelet ablations of vision transformer attention
maps~\cite{abraham2025wavelet}, and attention entropy collapse as a
deep-transformer failure mode~\cite{zhai2023stabilizing}. However,
wavelet decomposition has not previously been applied post hoc to
hidden-state trajectories of language models. The present work fills
this gap by applying the CWT to projected residual-stream
trajectories and deriving the conditional scale entropy.

\section{Methodology}
\label{sec:methodology}
This section presents our framework for analysing the spectral
structure of transformer computation. The primary tool is the
Continuous Wavelet Transform (CWT) applied to hidden-state
trajectories across layers; the application to metaphor processing
is described in Section~\ref{sec:experiments}.

\subsection{Problem Setting and Notation}
\label{sec:problem}

Consider a transformer language model with $L$ layers and embedding
dimension $d$. For a given input sentence, each token $\tau$ is
associated with a sequence of hidden-state vectors
$x_0^\tau, x_1^\tau, \ldots, x_L^\tau \in \mathbb{R}^d$, where
$x_0^\tau$ is the initial embedding and $x_L^\tau$ is the final-layer
representation. The residual connection \cite{vaswani2017attention}
yields the recurrence
\begin{equation}\label{eq:residual}
x_{l+1}^\tau = x_l^\tau + \Delta_l(x_l^\tau),
\end{equation}
where $\Delta_l = \mathrm{Attn}_l + \mathrm{MLP}_l$ is the combined
attention and feed-forward update at layer $l$. Following Elhage et
al.\ \cite{elhage2021mathematical}, we refer to the sequence
$(x_0^\tau, \ldots, x_L^\tau)$ as the \emph{residual trajectory} of
token $\tau$.

Our analysis is built on a \emph{minimal-pair} design. Given a target
lexeme $w$ (e.g., \emph{crushed}), we embed it in two sentences that
differ only in whether $w$ is used literally (\emph{The weight crushed
the box.}) or metaphorically (\emph{The news crushed her hopes.}).
The object of study is the pair of residual trajectories
$(x_0^{\mathrm{lit}}, \ldots, x_L^{\mathrm{lit}})$ and
$(x_0^{\mathrm{met}}, \ldots, x_L^{\mathrm{met}})$ for the shared
target token. We ask whether the spectral structure of their
respective layer-to-layer updates differs in a systematic and
measurable way.
\subsection{Projected Trajectory and the Contrast Direction}
\label{sec:projection}

The residual updates $\Delta_l^\tau \in \mathbb{R}^d$ are
high-dimensional vectors. Comparing them directly across conditions
requires reducing them to a representation that preserves the
metaphor--literal distinction while enabling frequency analysis. We
achieve this by projecting onto a single direction in $\mathbb{R}^d$.

Given $K$ minimal pairs, each contributing literal updates
$\{\Delta_l^{\mathrm{lit},k}\}_{l=0}^{L-1}$ and metaphorical updates
$\{\Delta_l^{\mathrm{met},k}\}_{l=0}^{L-1}$, we define the
\emph{normalized mean-difference direction}
\begin{equation}\label{eq:vstar}
v^* = \frac{\displaystyle\sum_{k=1}^{K}\sum_{l=0}^{L-1}
  \bigl(\Delta_l^{\mathrm{met},k} -
  \Delta_l^{\mathrm{lit},k}\bigr)}
{\displaystyle\bigl\|\sum_{k=1}^{K}\sum_{l=0}^{L-1}
  \bigl(\Delta_l^{\mathrm{met},k} -
  \Delta_l^{\mathrm{lit},k}\bigr)\bigr\|_2}.
\end{equation}
This is the unit vector along the aggregate update difference between
metaphorical and literal processing, summed across all pairs and
layers. The \emph{projected trajectory} is $s(l) = v^{*\top}x_l^\tau$
and the \emph{projected updates} are
$\delta_l = v^{*\top}\Delta_l$, collected into the vector
$\boldsymbol{\delta} \in \mathbb{R}^L$, yielding a scalar signal
indexed by layer depth that can be analysed with standard signal
processing tools. A leave-one-out test confirms that $v^*$ captures
a stable structural property: in every experiment, holding out a
pair and re-estimating $v^*$ from the remainder still separates
the held-out pair with the correct sign. Because $v^*$ is defined
relative to a particular set of pairs, we also test sensitivity by
re-estimating it from the VUA data in Section~\ref{sec:vua}.

Two quantities follow from the projection. \emph{Projection separation} ($\bar{\delta}_k^{\mathrm{met}} - \bar{\delta}_k^{\mathrm{lit}}$) is a first-moment check that $v^*$ captures a meaningful contrast. The conditional scale entropy (Section~\ref{sec:cse}) is a second-order structural quantity: two update vectors with identical mean projected differences can produce different scale entropy if their updates are organised differently across layers, so it asks whether the contrast reflects single-layer reconfiguration or multi-layer coordination.

\subsection{Wavelet Decomposition of the Projected Trajectory}
\label{sec:wavelet_decomp}

The projected trajectory $s(l)$ is a discrete signal of $L{+}1$
samples. We want to characterize its frequency content not globally
but at each layer position: at layer $b$, which frequency
scales carry energy? The Discrete Wavelet Transform decomposes a
signal into octave-spaced subbands but does not localize energy along
the signal's length, whereas the Continuous Wavelet Transform (CWT)
produces a two-dimensional scalogram that maps energy as a function
of both frequency scale $a$ and position $b$
\cite{mallat2009wavelet}, which is what the conditional scale entropy
(Section~\ref{sec:cse}) requires. We therefore use the CWT.

We embed the discrete signal $s(l)$ as a continuous function
$\tilde{s}(t)$ via piecewise-linear interpolation on $[0, L]$,
extended by constants $s(0)$ and $s(L)$ outside this interval, a
standard construction that avoids artificial discontinuities
\cite{Torrence1998APG}. We compute the CWT using the real Morlet
wavelet
\begin{equation}\label{eq:morlet}
\psi(t) = e^{-t^2/2}\cos(\omega_0 t), \qquad \omega_0 = 5,
\end{equation}
implemented by PyWavelets (\texttt{pywt.cwt} with \texttt{"morl"}).
The Morlet wavelet provides good joint time--frequency localization
and approximately annihilates constants
($|\hat{\psi}(0)| \approx 9.3 \times 10^{-6}$, verified numerically).
The CWT at scale $a > 0$ and position $b \in \mathbb{R}$ is
\begin{equation}\label{eq:cwt}
W_\psi \tilde{s}(a, b) = \frac{1}{\sqrt{a}}
\int_{-\infty}^{\infty}\tilde{s}(t)\,
\psi\!\left(\frac{t-b}{a}\right)dt,
\end{equation}
and its squared modulus $|W_\psi\tilde{s}(a,b)|^2$ is the
\emph{scalogram}. We evaluate the scalogram at $S$ integer scales $a_j = j$ for $j = 1, \ldots, S$, with $S = L+1$. The smallest scale $a_1 = 1$
captures features that span a single layer; the largest scale
$a_S = L+1$ matches the length of the trajectory itself. Scales
larger than this would extend beyond the signal and fall outside
the trustworthy region defined by the cone of influence~\cite{Torrence1998APG}.

Because $\tilde{s}$ is piecewise linear, the scalogram at any fixed
position $b_0$ is a linear function of the update vector
$\boldsymbol{\delta}$. Specifically, there exists a wavelet response
operator $\Phi_{b_0}\in\mathbb{R}^{S\times L}$, determined entirely
by the wavelet and layer geometry, such that
\begin{equation}\label{eq:operator}
\mathbf{z} \;=\; \Phi_{b_0}\,\boldsymbol{\delta} \in \mathbb{R}^S,
\qquad
|W_\psi\tilde{s}(a_j, b_0)|^2 \approx z_j^{\,2}.
\end{equation}
A useful consequence of this decomposition: when only one layer has a nonzero update, the normalized scale distribution $z_j^2/\|\mathbf{z}\|^2$ is independent of that update's magnitude, so any observed between-condition entropy difference reflects multi-layer interaction structure, not the magnitude of any single update.

\subsection{Conditional Scale Entropy}
\label{sec:cse}

The scalogram $|W_\psi\tilde{s}(a_j, b_k)|^2$ assigns an energy
value to each combination of scale $a_j$ and position $b_k$. At a
fixed position $b_k$, normalizing across scales produces a probability
distribution over frequency scales:
\begin{equation}\label{eq:csd}
p(a_j \mid b_k) = \frac{z_j^2}{\|\mathbf{z}\|^2},
\end{equation}
where $\mathbf{z} = \Phi_{b_k}\boldsymbol{\delta}$ as in
\eqref{eq:operator}. The \emph{conditional scale entropy} is the
Shannon entropy of this distribution:
\begin{equation}\label{eq:cse}
H(\mathrm{scale} \mid b_k) =
-\sum_{j=1}^{S} p(a_j \mid b_k)\,\log\,p(a_j \mid b_k).
\end{equation}
This quantity measures the breadth of frequency scales engaged by the
model's computation at layer position $b_k$. High values indicate that
energy is distributed across many scales simultaneously; low values
indicate concentration at a single scale. The structure of this
quantity is characterized by the following result.

\begin{theorem}[Structure of the Conditional Scale Entropy]
\label{thm:structure}
Let $\boldsymbol{\delta} \in \mathbb{R}^L$ and
$\mathbf{z} = \Phi_{b_0}\boldsymbol{\delta} \in \mathbb{R}^S$. Then:
\begin{enumerate}[label=(\roman*)]
\item $p(a_j \mid b_0) = z_j^2 / \|\mathbf{z}\|^2$.
\item $H(\mathrm{scale} \mid b_0) = \log S$ if and only if
$|z_1| = \cdots = |z_S|$.
\item $H(\mathrm{scale} \mid b_0) = 0$ if and only if exactly one
$z_j \neq 0$.
\item $H(\mathrm{scale} \mid b_0)$ is invariant under
$\boldsymbol{\delta} \mapsto c\boldsymbol{\delta}$ for any
$c \neq 0$.
\end{enumerate}
\end{theorem}

\begin{proof}

\noindent\textbf{Part (i).}\quad
By Definition~5.1, the conditional scale distribution at
position $b_0$ is
\begin{equation}\tag{A.1}
p(a_j \mid b_0)
\;=\;
\frac{|W_\psi\tilde{s}(a_j,\,b_0)|^2}
     {\displaystyle\sum_{j'=1}^{S}
      |W_\psi\tilde{s}(a_{j'},\,b_0)|^2}.
\end{equation}
From the wavelet response decomposition
(Equation~\eqref{eq:operator}),
\begin{equation}\tag{A.2}
W_\psi\tilde{s}(a_j,\,b_0)
\;=\;
\sum_{l=0}^{L-1}\delta_l\,\phi_l(a_j,\,b_0)
\;+\;
s(0)\sqrt{a_j}\,\hat{\psi}(0).
\end{equation}
The second term is the residual from constant annihilation.
For the real Morlet wavelet with $\omega_0 = 5$, we have
$|\hat{\psi}(0)| \approx 9.3 \times 10^{-6}$
(Section~\ref{sec:wavelet_decomp}). Working in the algebraic
formulation where constant annihilation is exact
($\hat{\psi}(0) = 0$), Equation~(A.2) reduces to
\begin{equation}\tag{A.3}
W_\psi\tilde{s}(a_j,\,b_0) \;=\; z_j,
\end{equation}
where $z_j = (\Phi_{b_0}\boldsymbol{\delta})_j$ is the $j$-th
component of $\mathbf{z}$. Since the real Morlet wavelet produces
real-valued coefficients, $|W_\psi\tilde{s}(a_j,\,b_0)|^2 = z_j^2$.
Substituting into (A.1):
\begin{equation}\tag{A.4}
p(a_j \mid b_0)
\;=\;
\frac{z_j^2}{\displaystyle\sum_{j'=1}^{S} z_{j'}^2}
\;=\;
\frac{z_j^2}{\|\mathbf{z}\|^2}.
\end{equation}

\medskip
\noindent\textbf{Part (ii).}\quad
From part~(i), define $p_j := z_j^2 / \|\mathbf{z}\|^2$.
This is a probability distribution on $S$ outcomes:
$p_j \geq 0$ for all $j$, and
\begin{equation}\tag{A.5}
\sum_{j=1}^{S} p_j
\;=\;
\frac{1}{\|\mathbf{z}\|^2}\sum_{j=1}^{S} z_j^2
\;=\;
\frac{\|\mathbf{z}\|^2}{\|\mathbf{z}\|^2}
\;=\; 1.
\end{equation}
The conditional scale entropy is
$H(\mathrm{scale} \mid b_0)
= -\sum_j p_j \log p_j = H(\mathbf{p})$.

\smallskip\noindent
$(\Longleftarrow)$\;
Suppose $|z_1| = |z_2| = \cdots = |z_S| =: \alpha$
for some $\alpha > 0$. Then $z_j^2 = \alpha^2$ for all $j$,
$\|\mathbf{z}\|^2 = S\alpha^2$, and
\begin{equation}\tag{A.6}
p_j \;=\; \frac{\alpha^2}{S\alpha^2} \;=\; \frac{1}{S}
\quad\text{for all } j.
\end{equation}
Therefore
\begin{equation}\tag{A.7}
H(\mathbf{p})
\;=\;
-\sum_{j=1}^{S}\frac{1}{S}\log\frac{1}{S}
\;=\;
-S \cdot \frac{1}{S} \cdot \log\frac{1}{S}
\;=\; \log S.
\end{equation}

\smallskip\noindent
$(\Longrightarrow)$\;
Suppose $H(\mathbf{p}) = \log S$.
By the maximum entropy theorem
\cite[Theorem~2.6.4]{cover2006elements},
$H(\mathbf{p}) \leq \log S$ for any distribution on $S$
outcomes, with equality if and only if $\mathbf{p}$ is
uniform. Therefore $p_j = 1/S$ for all $j$.
Substituting $p_j = z_j^2/\|\mathbf{z}\|^2 = 1/S$:
\begin{equation}\tag{A.8}
z_j^2 \;=\; \frac{\|\mathbf{z}\|^2}{S}
\quad\text{for all } j.
\end{equation}
Taking square roots:
$|z_j| = \|\mathbf{z}\|/\sqrt{S}$ for all $j$,
hence $|z_1| = \cdots = |z_S|$.

\medskip
\noindent\textbf{Part (iii).}\quad

\smallskip\noindent
$(\Longleftarrow)$\;
Suppose exactly one index $j_0$ has $z_{j_0} \neq 0$
and $z_j = 0$ for all $j \neq j_0$. Then
$\|\mathbf{z}\|^2 = z_{j_0}^2$, so
\begin{equation}\tag{A.9}
p_{j_0} \;=\; \frac{z_{j_0}^2}{z_{j_0}^2} \;=\; 1,
\qquad
p_j \;=\; \frac{0}{z_{j_0}^2} \;=\; 0
\quad\text{for } j \neq j_0.
\end{equation}
Therefore
$H(\mathbf{p})
= -1 \cdot \log 1
  - \sum_{j \neq j_0} 0 \cdot \log 0
= 0$,
where $0 \log 0 := 0$ by convention.

\smallskip\noindent
$(\Longrightarrow)$\;
Suppose $H(\mathbf{p}) = 0$.
Since $H(\mathbf{p}) \geq 0$ with equality if and only if
$\mathbf{p}$ is a point mass
\cite[Theorem~2.6.4]{cover2006elements},
there exists $j_0$ such that $p_{j_0} = 1$ and
$p_j = 0$ for $j \neq j_0$. Now:
\begin{align}
p_j = 0
&\;\Longrightarrow\;
z_j^2/\|\mathbf{z}\|^2 = 0
\;\Longrightarrow\;
z_j = 0,
\tag{A.10}\\[4pt]
p_{j_0} = 1
&\;\Longrightarrow\;
z_{j_0}^2 = \|\mathbf{z}\|^2
= \sum_{j=1}^{S} z_j^2.
\tag{A.11}
\end{align}
Subtracting $z_{j_0}^2$ from both sides of (A.11):
$\sum_{j \neq j_0} z_j^2 = 0$, confirming
$z_j = 0$ for all $j \neq j_0$.
Hence exactly one $z_j \neq 0$.

\medskip
\noindent\textbf{Part (iv).}\quad
Let $\boldsymbol{\delta}' = c\boldsymbol{\delta}$ for
$c \neq 0$. Define
$\mathbf{z}' = \Phi_{b_0}\boldsymbol{\delta}'$. By linearity:
\begin{equation}\tag{A.12}
\mathbf{z}'
\;=\;
\Phi_{b_0}(c\boldsymbol{\delta})
\;=\;
c\,\Phi_{b_0}\boldsymbol{\delta}
\;=\;
c\,\mathbf{z}.
\end{equation}
Therefore $z_j' = c\,z_j$ for every $j$. The new
probabilities are:
\begin{equation}\tag{A.13}
p_j'
\;=\;
\frac{(z_j')^2}{\|\mathbf{z}'\|^2}
\;=\;
\frac{(c\,z_j)^2}{\|c\,\mathbf{z}\|^2}
\;=\;
\frac{c^2\,z_j^2}{c^2\,\|\mathbf{z}\|^2}
\;=\;
\frac{z_j^2}{\|\mathbf{z}\|^2}
\;=\;
p_j.
\end{equation}
Since $p_j' = p_j$ for all $j$, the distribution is
unchanged, hence
$H(\mathrm{scale} \mid b_0)$ computed from
$\boldsymbol{\delta}'$ equals
$H(\mathrm{scale} \mid b_0)$ computed from
$\boldsymbol{\delta}$.
\end{proof}

Property~(iv) is central: CSE is insensitive to the overall magnitude
of $\boldsymbol{\delta}$ and responds only to the pattern of
updates across layers. Any observed entropy difference between two
conditions must therefore reflect multi-layer interaction structure,
not magnitude at any individual layer. A sufficient structural
condition for entropy ordering is provided by majorization.

\begin{theorem}[Majorization Implies Entropy Ordering]
\label{thm:major}
Let $\hat{z}_j^A = (z_j^A)^2/\|\mathbf{z}^A\|^2$ and
$\hat{z}_j^B = (z_j^B)^2/\|\mathbf{z}^B\|^2$. If
$\hat{\mathbf{z}}^A \prec \hat{\mathbf{z}}^B$ and
$\hat{\mathbf{z}}^A$ is not a permutation of
$\hat{\mathbf{z}}^B$, then
$H(\mathrm{scale}\mid b_0)^A > H(\mathrm{scale}\mid b_0)^B$.
\end{theorem}

\begin{proof}

\noindent\textbf{Step 1: Reduction to Shannon entropy.}\quad
By Theorem~\ref{thm:structure}(i), the conditional scale
distributions for signals $A$ and $B$ are
\begin{equation}\tag{A.14}
p_j^A
\;=\;
\frac{(z_j^A)^2}{\|\mathbf{z}^A\|^2}
\;=\;
\hat{z}_j^A,
\qquad
p_j^B
\;=\;
\frac{(z_j^B)^2}{\|\mathbf{z}^B\|^2}
\;=\;
\hat{z}_j^B.
\end{equation}
Both are valid probability distributions on $S$ outcomes:
$p_j^A \geq 0$ for all $j$ (being a ratio of squares),
and
$\sum_j p_j^A
= \sum_j (z_j^A)^2/\|\mathbf{z}^A\|^2
= \|\mathbf{z}^A\|^2/\|\mathbf{z}^A\|^2
= 1$;
likewise for $\mathbf{p}^B$.
The conditional scale entropies are therefore
\begin{equation}\tag{A.15}
H(\mathrm{scale} \mid b_0)^A = H(\hat{\mathbf{z}}^A),
\qquad
H(\mathrm{scale} \mid b_0)^B = H(\hat{\mathbf{z}}^B),
\end{equation}
where $H$ denotes Shannon entropy.

\medskip
\noindent\textbf{Step 2: Majorization.}\quad
For vectors $\mathbf{p},\,\mathbf{q} \in \mathbb{R}^S$ with
entries sorted in decreasing order
$p_{[1]} \geq p_{[2]} \geq \cdots \geq p_{[S]}$,
we write $\mathbf{p} \prec \mathbf{q}$
($\mathbf{p}$ is majorized by $\mathbf{q}$) if
\begin{equation}\tag{A.16}
\sum_{j=1}^{k} p_{[j]}
\;\leq\;
\sum_{j=1}^{k} q_{[j]}
\qquad
\text{for all } k = 1, \ldots, S{-}1,
\end{equation}
with equality at $k = S$
\cite[Chapter~1]{marshall2011inequalities}.

\medskip
\noindent\textbf{Step 3: Strict Schur-concavity of Shannon
entropy.}\quad
A function $f\colon \mathbb{R}^S \to \mathbb{R}$ is
\emph{strictly Schur-concave} if
$\mathbf{p} \prec \mathbf{q}$ and $\mathbf{p}$ is not a
permutation of $\mathbf{q}$ implies
$f(\mathbf{p}) > f(\mathbf{q})$.
Shannon entropy
$H(\mathbf{p}) = -\sum_{j} p_j \log p_j$
is strictly Schur-concave on the probability simplex
\cite[Chapter~3]{marshall2011inequalities}.

\medskip
\noindent\textbf{Step 4: Conclusion.}\quad
We are given $\hat{\mathbf{z}}^A \prec \hat{\mathbf{z}}^B$
and $\hat{\mathbf{z}}^A$ is not a permutation of
$\hat{\mathbf{z}}^B$. Applying Step~3 with
$\mathbf{p} = \hat{\mathbf{z}}^A$ and
$\mathbf{q} = \hat{\mathbf{z}}^B$:
\begin{equation}\tag{A.17}
H(\hat{\mathbf{z}}^A) \;>\; H(\hat{\mathbf{z}}^B).
\end{equation}
Substituting from (A.15):
\begin{equation}\tag{A.18}
H(\mathrm{scale} \mid b_0)^A
\;>\;
H(\mathrm{scale} \mid b_0)^B.
\end{equation}
\end{proof}

\medskip
\noindent\textbf{Remark} (one-directional implication).
The converse does not hold: entropy ordering does not
imply majorization. Consider
$\mathbf{p} = (0.5,\, 0.25,\, 0.25)$ and
$\mathbf{q} = (0.4,\, 0.4,\, 0.2)$. Direct computation
gives $H(\mathbf{q}) = 1.522$ bits $> 1.500$ bits $=
H(\mathbf{p})$. Yet the distributions are incomparable
under majorization:
$p_{[1]} = 0.5 > 0.4 = q_{[1]}$ violates the first
partial-sum condition for $\mathbf{p} \prec \mathbf{q}$,
while $q_{[1]} + q_{[2]} = 0.8 > 0.75 = p_{[1]} + p_{[2]}$
violates the condition for
$\mathbf{q} \prec \mathbf{p}$.
The significance of Theorem~\ref{thm:major} is that when
majorization does hold, the entropy ordering is guaranteed
under every Schur-concave function --- including
R\'{e}nyi entropies of all orders and the participation
ratio --- not only under Shannon entropy.

\subsection{The Multi-Scale Engagement Hypothesis}
\label{sec:hypothesis}

We now state the empirical prediction that
Sections~\ref{sec:cross_arch} and~\ref{sec:vua} test.

\paragraph{Multi-Scale Engagement Hypothesis.}
Metaphorical processing produces higher conditional scale entropy
than literal processing at layer positions where the model actively
reconfigures the target token's representation:
$H(\mathrm{scale} \mid b)^{\mathrm{met}} > H(\mathrm{scale} \mid b)^{\mathrm{lit}}.$

\noindent By Theorem~\ref{thm:structure}(iv) this inequality concerns
the pattern of updates across layers, not their magnitude;
Theorem~\ref{thm:major} provides a sufficient structural condition
via majorisation.

Since the reconfiguration depth is model-dependent, we identify it empirically. On the 25 controlled minimal pairs (Section~\ref{sec:cross_arch}) we locate the contiguous range of positions where CSE elevation is significant under cluster-based permutation correction, and call this range the \emph{active zone} $\mathcal{A}_M$. We test this zone with two independent replications: (i) 200 naturalistic VUA pairs (Section~\ref{sec:vua}) with $v^*$ re-estimated from scratch on the VUA data, and (ii) four additional decoder-only architectures spanning 124M to 20B parameters, each with its own $v^*$ and identical statistical criteria.

\section{Experiments and Results}
\label{sec:experiments}

This section evaluates the Multi-Scale Engagement Hypothesis across
five transformer architectures and two stimulus sets.

\subsection{Experimental Design}
\label{sec:design}

\paragraph{Models.}
We evaluate five decoder-only transformer architectures spanning two
orders of magnitude in parameter count (Table~\ref{tab:models}):
GPT-2 Small/Medium/Large \cite{radford2019language} within a single family for controlled
scaling, LLaMA-2 7B \cite{touvron2023llama2} as a generalisation test, and GPT-oss 20B \cite{openai2025gptoss} as a
Mixture-of-Experts model with 32 experts per layer. Hidden states
are extracted from pretrained models without fine-tuning. All CWT
computations use PyWavelets (\texttt{pywt.cwt} with \texttt{"morl"})
at $S = L{+}1$ integer scales.

\begin{table}[t]
\centering
\caption{Models used in the evaluation.}
\label{tab:models}
\scriptsize
\setlength{\tabcolsep}{4pt}
\begin{tabular}{lccclc}
\hline
\textbf{Model} & \textbf{Params} & \textbf{Layers} &
\textbf{Hidden} & \textbf{Architecture} & \textbf{Precision} \\
\hline
GPT-2 Small  & 124M & 12 & 768  & Dense decoder-only & FP32  \\
GPT-2 Medium & 355M & 24 & 1024 & Dense decoder-only & FP32  \\
GPT-2 Large  & 774M & 36 & 1280 & Dense decoder-only & FP16  \\
LLaMA-2 7B   & 7B   & 32 & 4096 & Dense decoder-only & FP16  \\
GPT-oss 20B  & 20B  & 24 & 2880 & MoE (32 experts)   & MXFP4 \\
\hline
\end{tabular}
\end{table}

\paragraph{Stimulus sets.}
The first stimulus set consists of 25 controlled minimal pairs in
which a target lexeme appears once in a literal sentence and once in
a metaphorical sentence (e.g., \emph{The weight crushed the box} /
\emph{The news crushed her hopes}); each pair involves a concrete
physical verb or noun used in a cross-domain mapping between a
physical source and an abstract target. The shared lexeme ensures
that the initial token embedding is near-identical across
conditions, so differences in the residual trajectory arise from
contextual processing.

The second stimulus set is drawn from the VUA All POS
corpus~\cite{steen2010method}, the standard benchmark for metaphor
identification, providing token-level metaphor annotations on
naturalistic English text spanning fiction, news, academic, and
conversational genres. We access it via the
\texttt{liyucheng/VUA20} dataset (160,154 annotated tokens) and
extract 200 pairs in which the same word, matched by its
dictionary form (lemma) so that inflected variants are treated as
the same target, appears in both a metaphorically and a literally
annotated sentence. We restrict to 5--35-token sentences with the
target in a non-edge position, take one pair per lemma, and use a
fixed random seed for reproducibility. Unlike the controlled pairs,
the VUA set is dominated by conventionalised metaphors whose
figurative meaning is lexically entrenched, providing a stringent
test of whether CSE remains sensitive when the metaphor signal is
weaker.

As a specificity control, we additionally construct 10
literal--literal pairs sharing a target word but differing in
contextual complexity (e.g., \emph{The knife cut the bread} /
\emph{The surgeon cut through the multiple layers of tissue}).
These pairs contain no metaphor and test whether CSE elevation
reflects general semantic processing difficulty
(Section~\ref{sec:complexity}).

\paragraph{Evaluation quantities.}

The primary evaluation quantity is the layer-resolved conditional
scale entropy difference
\begin{equation}\label{eq:delta_h}
\Delta H_M(b) = H(\mathrm{scale} \mid b)^{\mathrm{met}}
              - H(\mathrm{scale} \mid b)^{\mathrm{lit}},
\end{equation}
computed for each pair and each layer position $b$. Positive values
indicate broader scale engagement for metaphorical processing.

Projection separation serves as a validation check on the learned
contrast direction $v^*$: for each pair $k$, a consistently positive
mean projected update difference across layers confirms that $v^*$
captures a stable metaphor--literal distinction. This quantity
validates the projection but is not itself a test of the Multi-Scale
Engagement Hypothesis.

Additional spectral quantities (CWT energy, $H_W$, $H(q)$) are summarised in Section~\ref{sec:cross_arch} as supporting analyses.

\paragraph{Statistical testing.}
For the controlled experiments ($K = 25$ pairs), we use Monte Carlo sign-flip tests (10{,}000 permutations) at each scalogram position; for the VUA experiments ($K = 200$), one-sample $t$-tests on paired differences with Cohen's $d = \bar{d}/s_d$.

To correct for multiple comparisons across scalogram positions, we use cluster-based permutation testing~\cite{maris2007nonparametric}. The cluster statistic is the sum of test statistics over the largest contiguous cluster exceeding the one-sided threshold ($\alpha = 0.05$), evaluated against 5{,}000 sign-flipped null datasets; the cluster $p$-value is the fraction of null maxima meeting or exceeding the observed value.

\subsection{Validation of the Contrast Direction}
\label{sec:validation}

Three checks confirm that $v^*$ captures a stable metaphor--literal distinction on every model tested.
(i)~\emph{Projection separation} is positive for at least 21/25 controlled pairs on every architecture (all $p < 0.001$).
(ii)~\emph{Leave-one-out generalisation} on GPT-2 Medium shows that a direction learned from any 24 pairs separates the held-out pair with the correct sign for all 25 holdouts.
(iii)~\emph{Stability under expansion}: $v^*$ estimated from the original 10 pairs and from all 25 pairs has cosine similarity 0.997, indicating that the direction reflects the model rather than any particular stimulus subset.

\subsection{Conditional Scale Entropy Across Architectures}
\label{sec:cross_arch}

\paragraph{Primary finding.}
Figure~\ref{fig:cse_main} shows the layer-resolved conditional
scale entropy difference $\Delta H = H^{\mathrm{met}} - H^{\mathrm{lit}}$ on GPT-2 Medium using the full 25-pair controlled set. The
metaphorical condition produces consistently higher conditional scale entropy at positions 5--13, with 13 of 25 scalogram positions reaching $p < 0.05$ under Monte Carlo sign-flip tests. The peak occurs at position 9 ($p = 0.0008$, 18/25 pairs positive).
Cluster-based permutation testing confirms the main contiguous cluster at positions 5--13 ($p_{\mathrm{cluster}} = 0.007$), accounting for multiple comparisons across all 25 scalogram positions.

\begin{figure}[t]
\centering
\includegraphics[width=0.85\linewidth]{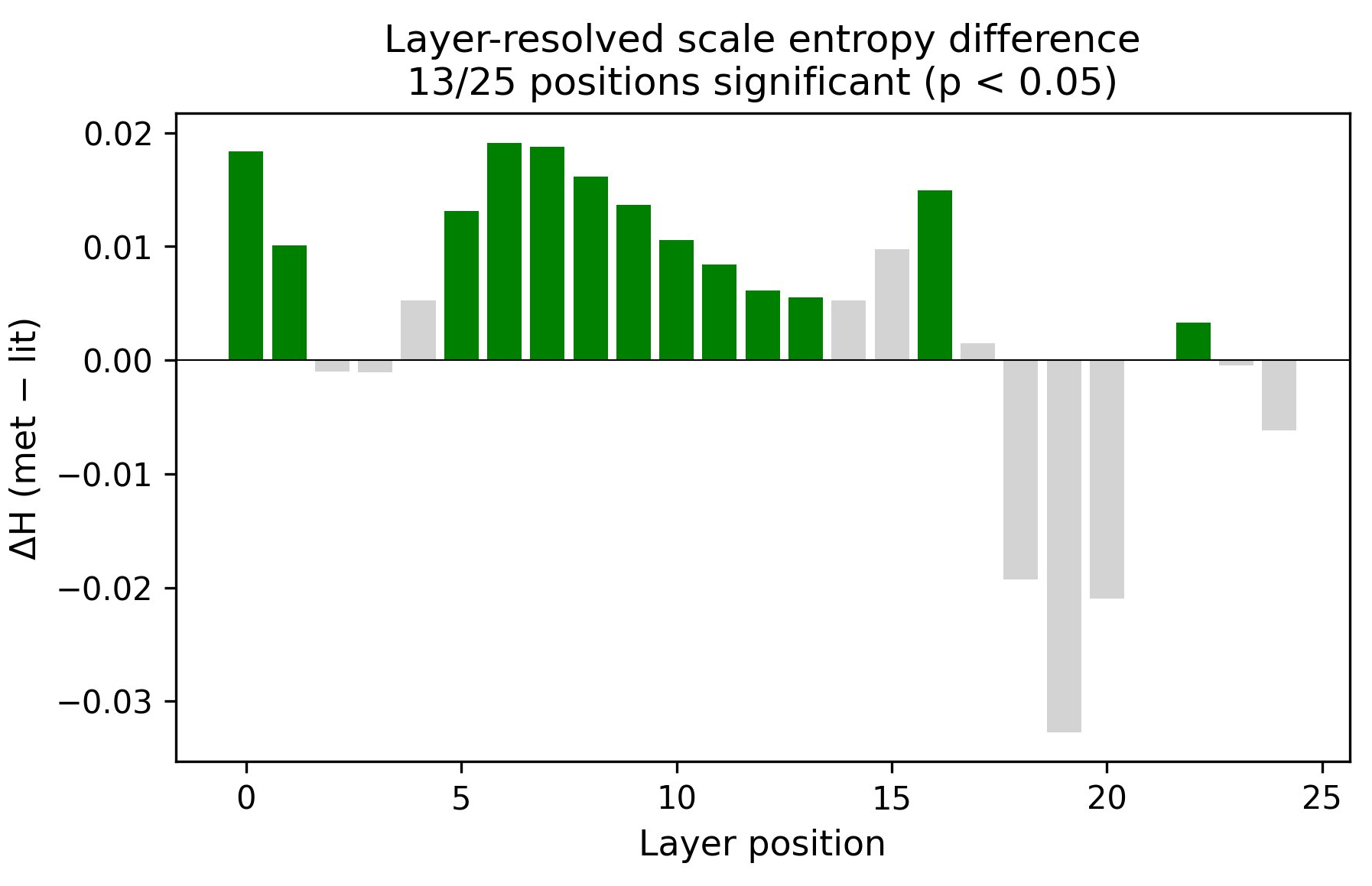}
\caption{Layer-resolved conditional scale entropy difference
$\Delta H = H^{\mathrm{met}} - H^{\mathrm{lit}}$ on GPT-2 Medium (25 controlled pairs). Green bars indicate positions reaching $p < 0.05$ under Monte Carlo sign-flip testing; the cluster at positions 5--13 survives cluster-based permutation correction ($p_{\mathrm{cluster}} = 0.007$). Absolute scale entropy spans $H \in [2.2, 3.1]$ across positions and is nearly identical between conditions, with the difference at the second decimal.}
\label{fig:cse_main}
\end{figure}

\paragraph{Cross-architecture consistency.}
Extending the analysis to four additional decoder-only architectures
with the same 25-pair set and identical wavelet parameters,
Table~\ref{tab:cse_cross} and Figure~\ref{fig:cse_cross} report the
per-architecture cluster statistics. On every model tested CSE
elevates significantly at a contiguous range of positions, and
every cluster survives cluster-based permutation correction at
$\alpha = 0.05$ ($p_{\mathrm{cluster}} < 0.001$ on four of five
models; $= 0.007$ on GPT-2 Medium). All active zones occupy the
early-to-mid relative-depth range, despite absolute depth varying
from 12 to 36 layers and architecture varying between dense
attention (GPT-2, LLaMA) and Mixture-of-Experts routing (GPT-oss).
R\'{e}nyi-2 conditional entropy on GPT-2 Medium produces identical
significant positions, ruling out sensitivity to
distributional tail behaviour.

\begin{table}[t]
\centering
\caption{Conditional scale entropy across architectures (25 controlled pairs,
cluster-based permutation correction, 5{,}000 permutations). Active zones
consistently occupy early-to-mid depth across all architectures.}
\label{tab:cse_cross}
\scriptsize
\setlength{\tabcolsep}{3pt}
\begin{tabular}{lcccccc}
\hline
\textbf{Model} & $K$ & \textbf{Entropy} & \textbf{Cluster (positions)}
& $p_{\mathrm{cluster}}$ & \textbf{Peak} $p$
& \textbf{Rel.\ depth} \\
\hline
GPT-2 Small  & 25 & Shannon   & 2--8   & $<$0.001 & $<$0.001 & 17--67\% \\
GPT-2 Medium & 25 & Shannon   & 5--13  & 0.007    & 0.0008   & 21--54\% \\
GPT-2 Large  & 25 & Shannon   & 0--12  & $<$0.001 & $<$0.001 & 0--33\% \\
LLaMA-2 7B   & 25 & Shannon   & 5--11  & $<$0.001 & $<$0.001 & 16--34\% \\
GPT-oss 20B  & 25 & Shannon   & 0--13  & $<$0.001 & $<$0.001 & 0--54\% \\
\hline
\end{tabular}
\end{table}

\begin{figure}[t]
\centering
\includegraphics[width=\textwidth]{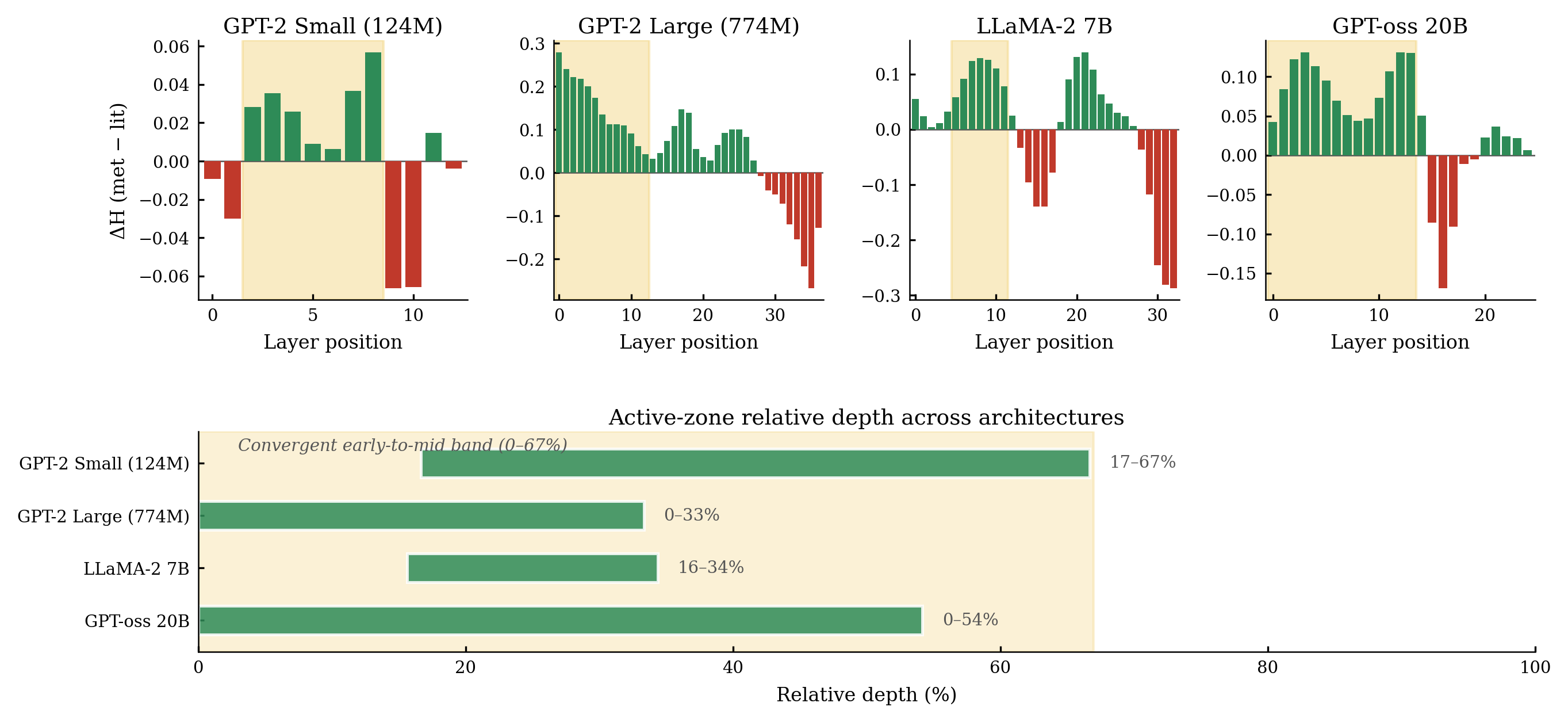}
\caption{Layer-resolved conditional scale entropy difference
$\Delta H = H^{\mathrm{met}} - H^{\mathrm{lit}}$ across four decoder-only
architectures (25 controlled pairs each; GPT-2 Medium shown separately
in Figure~\ref{fig:cse_main}). Green: $\Delta H > 0$; red: $\Delta H < 0$;
shaded region indicates the contiguous significant cluster
(cluster-based permutation test, 5{,}000 permutations, $\alpha = 0.05$).
Bottom panel: active-zone relative depth for each architecture. Despite
absolute depth varying from 12 to 36 layers, every cluster lies within
the early-to-mid relative-depth band (0--67\%).}
\label{fig:cse_cross}
\end{figure}

\paragraph{Robustness of CSE.}
Among the spectral quantities examined, only CSE generalises reliably. Energy suppression (CWT power, met~$<$~lit) is significant on GPT-2 Medium but absent on LLaMA-2 7B and GPT-oss 20B. More strikingly, update-energy entropy $H(q)$ reverses sign on the same model when the stimulus set is expanded from 10 to 25 pairs ($p = 0.003$, met~$<$~lit $\to$ $p = 0.002$, met~$>$~lit), a pattern mirrored across architectures; a quantity whose sign depends on stimulus composition cannot serve as a reliable signature. The attenuation of global metrics on LLaMA-2 7B parallels GPT-oss 20B, suggesting the effect is linked to model scale rather than the dense/MoE distinction. CSE is the only metric whose direction and significance are stable across pair expansion, architectures, and, as shown below, stimulus regimes.

\subsection{Generalisation to Naturally-Occurring Metaphor: the VUA Corpus}
\label{sec:vua}

The controlled experiments use carefully constructed stimuli; we now
test whether the conditional scale entropy generalises to naturalistic
metaphor drawn from an established benchmark. Stimulus construction
details are described in Section~\ref{sec:design}; all analyses use
GPT-2 Medium with identical wavelet parameters.

\paragraph{Projection separation.}
The contrast direction $v^*$ from the 25 controlled pairs yields
positive mean projected update differences for 107/200 VUA pairs
($p = 0.19$): a non-significant first-moment result, consistent
with $v^*$ not fully capturing predominantly conventionalised
metaphors. As noted in Section~\ref{sec:projection}, projection
separation tests average update magnitude along $v^*$; a null
mean shift does not preclude structural differences in the depth
profile.

\paragraph{CSE elevation and convergent validation.}
Two analyses are needed because $v^*$ from the controlled pairs may not perfectly align the depth profile for conventionalised VUA metaphors. With the controlled-pair $v^*$, the VUA cluster falls at positions 0--4 ($p_{\mathrm{cluster}} = 0.013$), shifted earlier than $\mathcal{A}_M = \{5, \ldots, 13\}$ (Figure~\ref{fig:vua_convergent}, left). Re-estimating $v^*$ from the VUA data alone (cosine similarity 0.964 with the controlled-pair direction) shifts the significant cluster to positions 5--10 ($p_{\mathrm{cluster}} = 0.003$), falling entirely within $\mathcal{A}_M$
(Figure~\ref{fig:vua_convergent}, right). Two independent datasets with independently estimated contrast directions thus converge on the same depth range, confirming that the active zone reflects a stable property of how GPT-2 Medium processes figurative language rather than an artifact of any one stimulus set.

\begin{figure}[t]
\centering
\includegraphics[width=\linewidth]{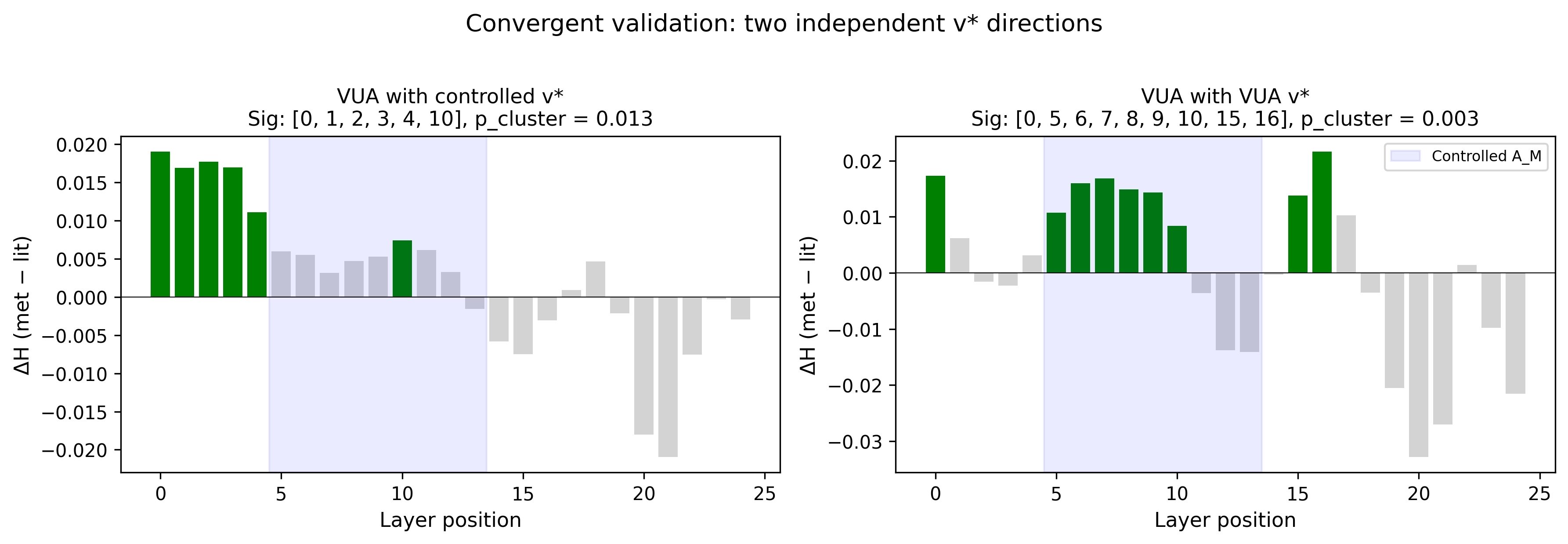}
\caption{Convergent validation of the VUA active zone on GPT-2
Medium (200 pairs). Left: VUA analysed with the controlled-pair
$v^*$ — the significant cluster falls at positions 0--4
($p_{\mathrm{cluster}} = 0.013$), shifted earlier than the
controlled active zone $\mathcal{A}_M = \{5, \ldots, 13\}$
(blue shading). Right: VUA analysed with $v^*$ independently
re-estimated from the VUA data — the significant cluster shifts
to positions 5--10 ($p_{\mathrm{cluster}} = 0.003$), falling
entirely within $\mathcal{A}_M$. Two independent datasets with
independently estimated contrast directions thus converge on the
same depth range. Green bars indicate positions reaching
$p < 0.05$.}
\label{fig:vua_convergent}
\end{figure}

\paragraph{Graded effect size.}
The effect size for VUA metaphors ($d \approx 0.17$ at the cluster
peak) is approximately half that of the controlled vivid metaphors
($d \approx 0.34$), consistent with the Graded Salience
Hypothesis~\cite{giora1997graded}: conventionalised metaphors, whose
figurative meaning is lexically entrenched, produce a weaker but
structurally identical CSE elevation compared to novel cross-domain
metaphors.

\paragraph{Robustness of the conditional scale entropy.}
On GPT-2 Medium across the two stimulus sets, projection
separation, CWT energy, $H_W$, and $H(q)$ all attenuate to
non-significance on the VUA data, while only the conditional
scale entropy generalises to naturalistic data, mirroring the
cross-architecture pattern of Section~\ref{sec:cross_arch}.

\subsection{Specificity Controls}
\label{sec:complexity}

Two controls test whether CSE elevation reflects the
metaphor--literal contrast specifically, rather than general
semantic processing difficulty or the particular content of
metaphorical sentences.

\paragraph{Semantic complexity (literal--literal).}
We construct 10 pairs of literal sentences sharing a target word
but differing in contextual complexity (e.g., \emph{The knife cut
the bread} / \emph{The surgeon cut through the multiple layers of
tissue}). Using the same contrast direction $v^*$ and identical
wavelet parameters, no scalogram position reaches significance
under sign-flip testing. Mean $\Delta H$ within the active zone
$\mathcal{A}_M = \{5, \ldots, 13\}$ is $-0.0005$ for the
complexity contrast versus $0.017$ for the metaphor contrast.
CSE is therefore specific to the metaphor--literal distinction
and does not respond to variation in semantic complexity within
literal language.

\paragraph{Expression mode (matched-meaning triples).}
To separate figurativeness from propositional content, we
construct 10 minimal triples in which two distinct metaphorical
sentences and one literal sentence express the same meaning in
identical syntactic frames (e.g., \emph{She is a lightning bolt} /
\emph{She is a speeding bullet} / \emph{She is very fast}). Pooling
both metaphors against the literal, CSE elevates significantly at
positions 4--16 ($p_{\text{cluster}} = 0.003$), overlapping
$\mathcal{A}_M$. Each metaphor individually produces nearly
identical CSE profiles against the literal (14/25 significant
positions each), while the two metaphors compared against each
other yield zero significant positions ($|\Delta H| < 0.001$),
confirming that figurative mode of expression, not propositional
content, drives CSE elevation.

\section{Discussion}
\label{sec:discussion}

\paragraph{Organisation, not intensity.}
The invariance property of Theorem~\ref{thm:structure}(iv) is
central: CSE isolates the organisation of computation across
scales from its intensity. Magnitude-dependent metrics
(CWT energy, update-energy concentration) produce
significant results on GPT-2 Medium but do not survive changes of
architecture, stimulus regime, or even stimulus expansion within a
single model: $H(q)$ reverses direction (met~$<$~lit to
met~$>$~lit) on GPT-2 Medium when the controlled set is expanded
from 10 to 25 pairs (see Section~\ref{sec:cross_arch}). CSE survives
all three tests. Metaphorical processing is therefore distinguished
not by greater effort at any single depth but by broader
coordination of representational updates across frequency scales.

\paragraph{Cross-depth structure as a complementary lens.}
Current mechanistic interpretability methods (activation patching,
sparse autoencoders, circuit tracing~\cite{elhage2021mathematical})
operate at the level of individual components: they address what
specific neurons or attention heads do, but not how computation is
organised across the depth of the residual stream. Sun et
al.~\cite{sun2025painters} showed that middle layers can be skipped
or reordered with limited performance loss; our results add nuance:
when the model encounters semantically complex input, the \emph{pattern}
of updates across layers changes even when individual layers appear
interchangeable on average. CSE captures this input-dependent
depth structure as a macro-level diagnostic complementary to
component-level analysis.

\paragraph{Convergent validation and graded effect.}
When $v^*$ is re-estimated by applying Equation~\eqref{eq:vstar} to the 200 VUA pairs (rather than the 25 controlled pairs), the resulting active zone (positions 5--10) falls within the controlled-pair zone (positions 5--13). The graded effect size
($d \approx 0.34$ for vivid cross-domain metaphors,
$d \approx 0.17$ for conventionalised VUA metaphors) is
consistent with the Graded Salience
Hypothesis~\cite{giora1997graded}: metaphors whose figurative
meaning is not yet lexically entrenched demand more extensive
multi-scale coordination. 

\paragraph{Architectural implications.}
The active zone consistently occupies the early-to-mid relative-depth
range on all five architectures, despite threefold variation in
absolute depth and a change of attention mechanism between dense and
Mixture-of-Experts. Situating this within the layer-uniformity
findings of Sun et al.~\cite{sun2025painters}, our results locate,
within the middle-layer plateau, the phase where cross-domain semantic
content is integrated --- in its early half, before the late-layer
prediction phase described by Geva et al.~\cite{geva2022ffn}. This
architecture-independence has a practical implication: interventions
targeting figurative-language processing are most likely to take
effect in the early-to-mid residual stream regardless of model.

\paragraph{Scope and future directions.}
The current experiments cover 25 controlled pairs, 200 naturalistic
VUA pairs, and five architectures up to 20B parameters; specificity
controls each use 10 pairs. Larger stimulus sets, frontier-scale
models, and a stimulus-independent contrast direction would
strengthen generalisability. Three extensions are of particular
interest: (i)~stimuli with continuously rated novelty, enabling a
finer-grained test of the graded prediction;
(ii)~causal intervention via scale-channel suppression using the
wavelet response operator (Equation~\eqref{eq:operator}), to
establish whether multi-scale coordination is functionally necessary
or merely correlated with figurative processing; and
(iii)~application to polysemy, irony, metonymy, and multi-step
reasoning~\cite{sravanthi2024pub}, to determine whether the
signature is specific to metaphor or reflects a broader
computational property of semantically complex input.

\section{Conclusion}
\label{sec:conclusion}

We introduced the conditional scale entropy (CSE), a wavelet-derived
measure of how broadly transformer computation engages across
frequency scales at each layer position. Two theorems establish that
CSE is invariant to update magnitude and ordered by majorisation of
the normalised scale distribution. Across five decoder-only
architectures from 124M to 20B parameters, metaphorical tokens
elevate CSE at a contiguous range of layer positions on every model,
with all clusters surviving cluster-based permutation correction and
falling within the early-to-mid relative-depth band. Convergent
validation on 200 naturalistic VUA pairs recovers the same active
zone and a graded effect-size pattern consistent with the Graded
Salience Hypothesis. Specificity controls confirm that the elevation
reflects figurative mode of expression rather than semantic
difficulty or propositional content. Among all spectral quantities
examined, CSE is uniquely stable across architectures, stimulus
expansion, and stimulus regimes, identifying multi-scale
coordination as a robust computational signature of metaphorical
language processing and positioning CSE as a principled analytical
tool for characterising how transformers organise computation across
depth.


%
%
\bibliographystyle{splncs04}
\bibliography{references}

@inproceedings{gao2018neural,
  title={Neural Metaphor Detection in Context},
  author={Gao, Ge and Choi, Eunsol and Choi, Yejin and Zettlemoyer, Luke},
  booktitle={Proceedings of EMNLP},
  pages={607--613},
  year={2018}
}

@inproceedings{mao2019end,
  title={End-to-End Sequential Metaphor Identification Inspired by Linguistic Theories},
  author={Mao, Rui and Lin, Chenghua and Guerin, Frank},
  booktitle={Proceedings of ACL},
  pages={3888--3898},
  year={2019}
}

@inproceedings{choi2021melbert,
  title={{MelBERT}: Metaphor Detection via Contextualized Late Interaction using Metaphorical Identification Theories},
  author={Choi, Minjin and Lee, Sunkyung and Choi, Eunseong and Park, Heesoo and Lee, Junhyuk and Lee, Dongwon and Lee, Jongwuk},
  booktitle={Proceedings of NAACL},
  pages={1763--1773},
  year={2021}
}

@book{steen2010method,
  title={A Method for Linguistic Metaphor Identification},
  author={Steen, Gerard J and Dorst, Aletta G and Herrmann, J Berenike and Kaal, Anna A and Krennmayr, Tina and Pasma, Trijntje},
  publisher={John Benjamins},
  year={2010}
}

@inproceedings{wachowiak2023does,
  title={Does {GPT-3} Grasp Metaphors? Identifying Metaphor Mappings with Generative Language Models},
  author={Wachowiak, Lennart and Gromann, Dagmar},
  booktitle={Proceedings of ACL (Volume 1:
    Long Papers).},
  pages={1018--1032},
  year={2023}
}

@inproceedings{tenney2019bert,
  title={{BERT} Rediscovers the Classical {NLP} Pipeline},
  author={Tenney, Ian and Das, Dipanjan and Pavlick, Ellie},
  booktitle={Proceedings of ACL},
  pages={4593--4601},
  year={2019}
}

@inproceedings{ethayarajh2019contextual,
  title={How Contextual are Contextualized Word Representations? Comparing the Geometry of {BERT}, {ELMo}, and {GPT-2} Embeddings},
  author={Ethayarajh, Kawin},
  booktitle={Proceedings of EMNLP},
  pages={55--65},
  year={2019}
}

@article{elhage2021mathematical,
  title={A Mathematical Framework for Transformer Circuits},
  author={Elhage, Nelson and Nanda, Neel and Olsson, Catherine and Henighan, Tom and Joseph, Nicholas and Mann, Ben and Askell, Amanda and Bai, Yuntao and Chen, Anna and Conerly, Tom and others},
  journal={Transformer Circuits Thread},
  year={2021}
}

@inproceedings{aghazadeh2022metaphors,
  title={Metaphors in Pre-Trained Language Models: Probing and Generalization Across Datasets and Languages},
  author={Aghazadeh, Ehsan and Fayyaz, Mohsen and Yaghoobzadeh, Yadollah},
  booktitle={Proceedings of ACL},
  pages={2037--2050},
  year={2022}
}

@inproceedings{voita2020information,
  title={Information-Theoretic Probing with Minimum Description Length},
  author={Voita, Elena and Titov, Ivan},
  booktitle={Proceedings of EMNLP},
  pages={183--196},
  year={2020}
}

@inproceedings{rahaman2019spectral,
  title={On the Spectral Bias of Neural Networks},
  author={Rahaman, Nasim and Barber, Aristide and Arpit, Devansh and Draxler, Felix and Lin, Min and Hamprecht, Fred and Bengio, Yoshua and Courville, Aaron},
  booktitle={Proceedings of ICML},
  pages={5301--5310},
  year={2019}
}

@article{shwartzziv2017opening,
  title={Opening the Black Box of Deep Neural Networks via Information},
  author={Shwartz-Ziv, Ravid and Tishby, Naftali},
  journal={arXiv preprint arXiv:1703.00810},
  year={2017}
}

@inproceedings{abraham2025wavelet,
  title={Wavelet-Based Mechanistic Interpretability of Vision Transformers},
  author={Abraham, Joshua and others},
  booktitle={CVPR Workshop on Mechanistic Interpretability in Vision},
  year={2025}
}

@article{zhai2023stabilizing,
  title={Stabilizing Transformer Training by Preventing Attention Entropy Collapse},
  author={Zhai, Shuangfei and Likhomanenko, Tatiana and Littwin, Etai and Busbridge, Dan and Ramapuram, Jason and Zhang, Yizhe and Gu, Jiatao and Susskind, Josh},
  journal={arXiv preprint arXiv:2303.08296},
  year={2023}
}

@article{geshkovski2025mathematical,
  title={A Mathematical Perspective on Transformers},
  author={Geshkovski, Borjan and Letrouit, Cyril and Polyanskiy, Yury and Rigollet, Philippe},
  journal={Bulletin of the American Mathematical Society},
  volume={62},
  number={3},
  year={2025}
}

@article{giora1997graded,
  title={Understanding Figurative and Literal Language: The Graded Salience Hypothesis},
  author={Giora, Rachel},
  journal={Cognitive Linguistics},
  volume={8},
  number={3},
  pages={183--206},
  year={1997}
}

@article{bowdle2005career,
  title={The Career of Metaphor},
  author={Bowdle, Brian F and Gentner, Dedre},
  journal={Psychological Review},
  volume={112},
  number={1},
  pages={193--216},
  year={2005}
}

@inproceedings{geva2022ffn,
  title={Transformer Feed-Forward Layers Build Predictions by Promoting Concepts in the Vocabulary Space},
  author={Geva, Mor and Caciularu, Avi and Wang, Kevin and Goldberg, Yoav},
  booktitle={Proceedings of the 2022 Conference on Empirical Methods in Natural Language Processing},
  pages={30--45},
  year={2022}
}

@article{Torrence1998APG,
  title={A Practical Guide to Wavelet Analysis.},
  author={Christopher Torrence and Gilbert P. Compo},
  journal={Bulletin of the American Meteorological Society},
  year={1998},
  volume={79},
  pages={61-78},
  url={https://api.semanticscholar.org/CorpusID:14928780}
}

@book{mallat2009wavelet,
  author    = {Mallat, St\'{e}phane},
  title     = {A Wavelet Tour of Signal Processing: The Sparse Way},
  edition   = {3rd},
  year      = {2009},
  publisher = {Academic Press},
  isbn      = {978-0-12-374370-1}
}

@book{cover2006elements,
  author    = {Cover, Thomas M. and Thomas, Joy A.},
  title     = {Elements of Information Theory},
  edition   = {2nd},
  year      = {2006},
  publisher = {Wiley-Interscience}
}

@book{marshall2011inequalities,
  author    = {Marshall, Albert W. and Olkin, Ingram and
               Arnold, Barry C.},
  title     = {Inequalities: Theory of Majorization and
               Its Applications},
  edition   = {2nd},
  year      = {2011},
  publisher = {Springer}
}

@inproceedings{sun2025painters,
  author    = {Sun, Qi and Pickett, Marc and Nain, Aakash Kumar
               and Jones, Llion},
  title     = {Transformer Layers as Painters},
  booktitle = {Proceedings of the AAAI Conference on Artificial
               Intelligence},
  volume    = {39},
  pages     = {25219--25227},
  year      = {2025},
}

@inproceedings{sravanthi2024pub,
  author    = {Sravanthi, Pragnya and Mamidi, Radhika},
  title     = {{PUB}: A Pragmatics Understanding Benchmark for
               Assessing {LLMs}' Pragmatics Capabilities},
  booktitle = {Proceedings of the 2024 Joint International
               Conference on Computational Linguistics, Language
               Resources and Evaluation (LREC-COLING)},
  year      = {2024},
}

@article{maris2007nonparametric,
  author  = {Maris, Eric and Oostenveld, Robert},
  title   = {Nonparametric statistical testing of {EEG}-
             and {MEG}-data},
  journal = {Journal of Neuroscience Methods},
  volume  = {164},
  number  = {1},
  pages   = {177--190},
  year    = {2007},
  doi     = {10.1016/j.jneumeth.2007.03.024},
}

@inproceedings{vaswani2017attention,
  author    = {Vaswani, Ashish and Shazeer, Noam and Parmar, Niki
               and Uszkoreit, Jakob and Jones, Llion and Gomez,
               Aidan N. and Kaiser, {\L}ukasz and Polosukhin,
               Illia},
  title     = {Attention is All You Need},
  booktitle = {Advances in Neural Information Processing Systems},
  volume    = {30},
  year      = {2017},
}

@techreport{radford2019language,
  title       = {Language Models are Unsupervised Multitask Learners},
  author      = {Radford, Alec and Wu, Jeffrey and Child, Rewon and
                 Luan, David and Amodei, Dario and Sutskever, Ilya},
  institution = {OpenAI},
  year        = {2019},
  url         = {https://cdn.openai.com/better-language-models/language_models_are_unsupervised_multitask_learners.pdf}
}

@article{touvron2023llama2,
  title         = {Llama 2: Open Foundation and Fine-Tuned Chat Models},
  author        = {Touvron, Hugo and Martin, Louis and Stone, Kevin and
                   Albert, Peter and Almahairi, Amjad and Babaei, Yasmine and
                   Bashlykov, Nikolay and Batra, Soumya and Bhargava, Prajjwal and
                   Bhosale, Shruti and others},
  journal       = {arXiv preprint arXiv:2307.09288},
  year          = {2023},
  doi           = {10.48550/arXiv.2307.09288}
}

@misc{openai2025gptoss,
  title         = {{gpt-oss-120b} \& {gpt-oss-20b} Model Card},
  author        = {{OpenAI}},
  year          = {2025},
  eprint        = {2508.10925},
  archivePrefix = {arXiv},
  primaryClass  = {cs.CL},
  url           = {https://arxiv.org/abs/2508.10925}
}

\end{document}